\documentclass[10pt]{article}
\PassOptionsToPackage{dvipsnames}{xcolor}

\AtBeginDocument{%
    \abovedisplayskip=8pt plus 2pt minus 4pt
    \belowdisplayskip=8pt plus 2pt minus 4pt
    \abovedisplayshortskip=2pt plus 2pt
    \belowdisplayshortskip=4pt plus 2pt minus 2pt
}

\usepackage[accepted]{rlj}

\usepackage{amssymb}
\usepackage{amsthm}
\usepackage{mathtools}
\usepackage{mathrsfs}
\usepackage{graphicx}
\usepackage{wrapfig}
\usepackage{subcaption}
\usepackage{url}
\usepackage{booktabs}
\usepackage{longtable}
\usepackage{array}
\usepackage{microtype}
\usepackage{aliascnt}
\usepackage[capitalise,noabbrev]{cleveref}

\renewcommand{\mid}{\,\vert\,}

\theoremstyle{plain}
\newtheorem{axiom}{Axiom}
\crefname{axiom}{Axiom}{Axioms}

\newtheorem{theorem}{Theorem}[section]

\newaliascnt{lemma}{theorem}
\newtheorem{lemma}[lemma]{Lemma}
\aliascntresetthe{lemma}
\crefname{lemma}{Lemma}{Lemmas}

\newaliascnt{proposition}{theorem}
\newtheorem{proposition}[proposition]{Proposition}
\aliascntresetthe{proposition}
\crefname{proposition}{Proposition}{Propositions}

\newaliascnt{corollary}{theorem}
\newtheorem{corollary}[corollary]{Corollary}
\aliascntresetthe{corollary}
\crefname{corollary}{Corollary}{Corollaries}

\theoremstyle{definition}
\newaliascnt{definition}{theorem}

\aliascntresetthe{definition}
\crefname{definition}{Definition}{Definitions}

\theoremstyle{remark}

\title{Rationalizing Boltzmann Rationality: An Axiomatic\\ Characterization of Entropy-Regularized Policies}

\setrunningtitle{Rationalizing Boltzmann Rationality}

\author{Silviu Pitis}

\emails{spitis@umich.edu}

\affiliations{
University of Michigan
}

\contribution{
    An axiomatic characterization (\cref{thm:boltzmann}): by restricting VNM Independence to base prospects (deterministic policies with environmental randomness), we show that IIA and monotonicity---applied to both policy and value at choice nodes---uniquely determine the Boltzmann policy, its entropy-regularized representation, and the soft Bellman equation.
    }
    {
    The softmax form has been derived through multiple independent axiom systems (\cref{sec:convergence}), but none provides a complete axiom-to-algorithm path in the MDP setting, where the RL literature has adopted softmax without axiomatic justification. The base-prospect domain restriction, the axiom decomposition into policy and value components, and the derivation of the entropy bonus (main proof, Step~2) are new.
    }

\contribution{
    RL-specific consequences: under generalized discounting with amplifying actions, a minimum rationality threshold $\beta_0$ emerges below which the soft Bellman equation has no finite fixed point---which is resolved by the horizon continuity axiom of prior work (\cref{prop:diverge,prop:horizon}). Actual return is non-decreasing in $\beta$ (\cref{prop:monotonicity}). A numerical example across the Tsallis entropy family illustrates that only Shannon entropy satisfies IIA.
    }
    {
    Prior work derives $\gamma(s,a)$ from preferences and shows that policy evaluation under amplifying actions is well-posed given horizon continuity \citep{pitis2019rethinking}; the interaction with entropy regularization was unexplored. The $\beta_0$ threshold and its resolution via horizon continuity are new; the return monotonicity result is standard, and included for completeness.
    }

\contribution{
    A cross-disciplinary synthesis connecting economics and information-theory axiomatics (Luce, Yellott, Shore--Johnson, McFadden, Fudenberg--Iijima--Strzalecki, Rust, et al.) to RL foundations, with a normative assessment of when IIA is appropriate for agent design.
    }
    {
    The foundational results are sparsely cited in the RL literature. The synthesis and normative assessment may serve as a useful reference for the RL community.
    }

\keywords{Axiomatic RL, Decision Theory, Stochastic Policies, Independence of Irrelevant Alternatives, Entropy Regularization, Boltzmann Rationality.}

\summary{
The softmax policy $\pi(a \mid s) \propto \exp(\beta Q(s,a))$ is the default model of stochastic choice in reinforcement learning (RL). Various justifications based on robustness, exploration, and optimization have been offered in the RL literature, but none uniquely derives the softmax form from first principles. This leaves a basic tension unresolved: the entropy bonus in the soft Bellman equation violates the Independence axiom that underwrites the Markov decision process (MDP) reward structure.
We dissolve this tension by distinguishing two kinds of randomness: chance and choice.
By restricting von Neumann--Morgenstern (VNM) Independence to environmental lotteries over base prospects, we show that imposing independence of irrelevant alternatives (IIA) and monotonicity on the policy and value functions at choice nodes uniquely determines the Boltzmann policy, the entropy-regularized representation, and the soft Bellman equation.
The choice between the soft and hard Bellman equations thus reduces to a design decision: whether the agent values its own ability to choose.
We develop RL-specific consequences, including return monotonicity and convergence under generalized discounting, and synthesize the independent lines from economics and information theory that arrive at the same structure, offering a normative assessment of when IIA is appropriate for agent design.
}

\begin{document}

\makeCover
\maketitle

\begin{abstract}
The softmax policy $\pi(a \mid s) \propto \exp(\beta Q(s,a))$ is the default model of stochastic choice in reinforcement learning (RL). Various justifications based on robustness, exploration, and optimization have been offered in the RL literature, but none uniquely derives the softmax form from first principles. This leaves a basic tension unresolved: the entropy bonus in the soft Bellman equation violates the Independence axiom that underwrites the Markov decision process (MDP) reward structure.
We dissolve this tension by distinguishing two kinds of randomness: chance and choice.
By restricting von Neumann--Morgenstern (VNM) Independence to environmental lotteries over base prospects, we show that imposing independence of irrelevant alternatives (IIA) and monotonicity on the policy and value functions at choice nodes uniquely determines the Boltzmann policy, the entropy-regularized representation, and the soft Bellman equation.
The choice between the soft and hard Bellman equations thus reduces to a design decision: whether the agent values its own ability to choose.
We develop RL-specific consequences, including return monotonicity and convergence under generalized discounting, and synthesize the independent lines from economics and information theory that arrive at the same structure, offering a normative assessment of when IIA is appropriate for agent design.
\end{abstract}

\section{Introduction}
\label{sec:introduction}

Boltzmann rationality---the softmax policy $\pi(a \mid s) \propto \exp({\beta Q(s,a)})$---is the default model of stochastic choice in reinforcement learning (RL). 
It underlies Boltzmann exploration \citep{sutton2018reinforcement}, soft actor-critic and entropy-regularized RL \citep{haarnoja2018soft}, the Bradley-Terry preference model in reinforcement learning from human feedback (RLHF) \citep{bradley1952rank}, maximum-entropy inverse RL \citep{ziebart2010modeling}, and human behavior modeling \citep{laidlaw2022boltzmann}.
The RL community has offered various justifications for this choice---translation invariance \citep{sutton2018reinforcement}, robustness \citep{eysenbach2022maximum}, exploration and optimization smoothing \citep{ahmed2019understanding}, the control-as-inference framework \citep{levine2018reinforcement}---but none that derives the softmax form from normative first principles and addresses the basic tension between stochastic choice and expected utility theory.

This paper resolves this tension by distinguishing chance from choice: Von Neumann-Morgenstern (VNM) Independence applies to environmental lotteries (chance), while Independence of Irrelevant Alternatives (IIA) and monotonicity govern behavioral selection (choice). We show that together, these axioms uniquely determine the Boltzmann policy, its entropy-regularized representation, and the soft Bellman equation. Our formulation embeds the result in infinite-horizon MDPs with state-action-dependent discounting $\gamma(s,a)$ \citep{white2017UnifyingTS}, and develops the RL-specific implications.

The axiomatic characterization matters because it justifies softmax as a design choice. The core axiom, IIA, in its familiar form, is a consistency requirement on the policy: an agent's relative preference between two actions should not depend on what other actions are available.
For example, the availability of an irrelevant action distorts relative preferences under $\varepsilon$-greedy but not under Boltzmann.
When this property is appropriate (e.g., actions are independent), the Boltzmann policy is not merely mathematically or empirically convenient, but uniquely justified. When IIA breaks down (e.g., redundant or hierarchically structured actions), the axiomatization pinpoints the failure mode rather than leaving it as an empirical surprise.
The entropy bonus has a concrete interpretation in this framework: it is the option premium---the value of being able to choose an action versus having it be externally imposed.

A second, value-level application of IIA prices this premium:  Value IIA (\cref{ax:iia}b) requires an action's choice probability to be a menu-independent function of the loss in value caused by removing that action. Under monotonicity, it can equivalently be read backward: removing actions chosen with the same probability must cause the same loss in menu value. With the remaining axioms, this forces the premium to be exactly the policy's entropy, $\tfrac{1}{\beta}H(\pi)$: large when choice is spread over comparable actions, approaching zero as one action dominates. Value IIA is, to our knowledge, new; it is also a strong assumption, doing much of the work in our characterization (it drives Step~2 of the main proof). \Cref{sec:alternative} therefore examines it from several angles, including an equivalent (given other axioms) \emph{marginal consistency} condition, $\partial V/\partial Q(s,a) = \pi(a \mid s)$.

\textbf{Contributions.}

\begin{enumerate}
\item
  An axiomatic characterization (\Cref{thm:boltzmann}): a domain restriction to base prospects reconciles VNM expected utility with entropy bonuses by separating environmental randomization from behavioral selection. 
  Within this framework, IIA and monotonicity axioms uniquely determine the Boltzmann policy, entropy-regularized representation, and soft Bellman equation, yielding a complete axiom-to-algorithm path in the RL setting. 
  Step~1 is classical \citep{luce1959individual}. Step~2 provides a novel derivation of the entropy bonus from Value IIA by comparing a menu's cardinal value before and after an action is removed. \Cref{sec:alternative} presents two alternative axiomatizations.
\item
  RL-specific discussion and consequences: under generalized discounting with $\gamma(s,a) > 1$, a minimum rationality threshold emerges below which the soft Bellman equation diverges, resolved by the horizon continuity axiom of prior work (\cref{prop:diverge,prop:horizon}). Actual return is non-decreasing in $\beta$ (\cref{prop:monotonicity}).
\item
  A cross-disciplinary synthesis connecting economics and information-theory axiomatics to RL foundations (SAC, control-as-inference, RLHF), with a normative assessment of when IIA is and is not appropriate for agent design.
\end{enumerate}

\section{Chance, Choice, and the Entropy Bonus}\label{sec:motivation}

In fully observed MDPs, an optimal deterministic policy always exists, and the value of any stochastic policy is an interpolation of its deterministic components---precisely what VNM rationality prescribes. And yet the practical advantages of stochastic policies are well established: under partial observability, stochastic stationary policies can strictly dominate deterministic ones \citep{singh1994learning}; stochasticity aids exploration; and entropy regularization improves learning speed and robustness (\cref{sec:introduction}).
A normative approach to stochastic choice should resolve the conflict with expected utility theory, and justify the choice of stochasticity (why softmax?). Existing justifications do neither. Indeed, any convex regularizer yields a valid soft Bellman equation \citep{geist2019theory}, and alternatives to Shannon entropy \citep{lee2018sparse,nachum2018tsallis} have been actively explored.

To see why such an approach requires care, consider that VNM's Independence axiom requires the value of any lottery to equal the weighted average of its components, which implies the \emph{compound lottery reduction} (CLR): $V(s) = \sum_a \pi(a \mid s)\, Q(s,a)$. The soft Bellman equation adds an entropy bonus $\frac{1}{\beta}H(\pi)$ that strictly exceeds this, violating Independence at every choice node. The incompatibility between softmax policies and CLR is plain: adding an option to the choice set can \emph{decrease} the agent's value. The question is not whether bounded agents deviate from the VNM prescription---they do---but whether a principled account of the entropy bonus exists within expected utility theory.

Our resolution rests on a distinction between two kinds of randomness in an MDP that the standard formalism conflates. At a state transition, or \emph{chance node}, the outcome is imposed on the agent, and Independence is the natural principle: the value of the lottery equals the weighted average of its components.
\looseness=-1 But at a \emph{choice node}, the agent has something more than a lottery: it has \emph{options}, which preserve the ability to adapt and can be worth more than any fixed lottery over their outcomes \citep{kreps1979representation}. In a continual learning setting \citep{abel2023definition}, for example, an agent at a choice node can adapt its behavior to current estimates, whereas one facing an externally imposed lottery cannot. 

The entropy bonus captures this option premium: the value of being at a decision node where the agent selects, versus having an outcome externally imposed.
Consider a stochastic policy and a lottery that produce identical distributions over outcomes: the agent should weakly prefer the stochastic policy, because being the source of randomness preserves the ability to deviate. \Cref{sec:axioms} formalizes this by restricting VNM Independence to environmental lotteries over base prospects and imposing IIA and monotonicity on both the policy and value function at choice nodes.

\section{From Preferences to the Soft Bellman Equation}
\label{sec:axioms}

\subsection{Preferences over Base Prospects}
\label{sec:baseprospects}

We work with finite, Markovian Decision Processes (MDPs), where $\mathcal{S}$ is a finite state space, $\mathcal{A}$ is a finite action space with $\vert\mathcal{A}\vert \geq 3$ (with two actions, IIA is trivially satisfied), and $P(s' \mid s, a)$ gives transition probabilities. The reward function $R: \mathcal{S} \times \mathcal{A} \to \mathbb{R}$ and state-action-dependent discount function $\gamma: \mathcal{S} \times \mathcal{A} \to [0, \bar{\gamma}]$ are representations of the agent's preferences, and are derived from axioms as discussed below. We write $H(\pi) = -\sum_a \pi_a \log \pi_a$ for entropy, $\text{softmax}(\beta q)_a = \exp({\beta q_a}) / \sum_{a'} \exp({\beta q_{a'}})$, and $\text{LSE}(q) = \log \sum_a \exp({q_a})$ for the log-sum-exp.

As previewed in \cref{sec:motivation}, we resolve the CLR incompatibility by restricting the scope of standard rationality axioms to environmental randomization. A \textbf{base prospect} is a process $p = (s, \Pi)$ with a deterministic, possibly non-Markovian policy $\Pi$, where all randomness is sourced from $P(\cdot\mid s,a)$. A \textbf{lottery} $\tilde p$ over base prospects is an external randomization over deterministic plans, selected by nature prior to execution. 
This is the paper's central modeling choice: the scope of Independence determines whether entropy bonuses are a violation or a consequence of the axioms.

\begin{axiom}[VNM Expected Utility]
\label{ax:vnm}
Preferences $\succ$ over lotteries of base prospects satisfy asymmetry, negative transitivity, independence, and continuity. Equivalently, there exists a utility $V$ such that for lotteries
$\tilde{p}, \tilde{q}$ over base prospects, $\tilde{p} \succ \tilde{q}$ if and only if $\mathbb{E}_{\tilde{p}}[V] > \mathbb{E}_{\tilde{q}}[V]$, with $V$ unique up to positive affine transformation. \citep{vonneumann1947theory,kreps1988notes}
\end{axiom}

\begin{axiom}[Dynamic Consistency (DC)]
\label{ax:dc}
Let $\Pi$ and $\Omega$ be arbitrary deterministic (possibly non-Markovian) policies, and let $a\Pi$ be the policy that first takes action $a$ and follows policy $\Pi$ thereafter. Then $(s, a\Pi) \succ (s, a\Omega)$ if and only if $\mathbb{E}_{s' \sim P(\cdot|s,a)}[V(s', \Pi)] >
\mathbb{E}_{s' \sim P(\cdot|s,a)}[V(s', \Omega)]$.
\end{axiom}
\vspace{-0.2em}

\Cref{ax:vnm,ax:dc} yield the generalized Bellman relation: there exist $R: \mathcal{S} \times \mathcal{A} \to \mathbb{R}$ and $\gamma: \mathcal{S} \times \mathcal{A} \to \mathbb{R}^+$ such that for all base prospects $(s,a\Pi)$, $V(s, a\Pi) = R(s,a) + \gamma(s,a)\,\mathbb{E}_{s'}[V(s', \Pi)]$ \citep[Theorem~3]{pitis2019rethinking}. As base prospects involve only environmental lotteries, the proof applies unchanged. 

Since $R(s,a)$ and $\gamma(s,a)$ are independent of the continuation policy, we define Q-values as $Q(s,a) \coloneqq R(s,a) + \gamma(s,a)\,\mathbb{E}_{s'}[V(s')]$, where $V(s')$ is the value of being in state $s'$. This is well defined for base prospects with deterministic continuation policies, but left undetermined for stochastic policies. It is this freedom that allows us to reconcile the soft Bellman equation with expected utility theory.

Here we assume that the Bellman relation extends from deterministic continuation policies to successor choice nodes: nature's lottery $P(\cdot \mid s,a)$ is aggregated affinely using the same $(R, \gamma)$, with each successor assigned the choice-node value $V(s')$ characterized below. Without this extension, \cref{ax:vnm,ax:dc} determine the Bellman relation only for deterministic continuations. The extension preserves the chance/choice distinction because affine aggregation applies to the environmental transition, not the agent's randomization within a choice node.

\subsection{Two Axioms for Stochastic Choice}

At each state $s$, given Q-values $Q(s, \cdot)$, the agent must select an action. We impose two axioms on a single choice rule $\pi: \mathbb{R}^{|\mathcal{A}|} \to \Delta(\mathcal{A})$, defined for every Q-vector and applied uniformly across states, and on the value function $V$, where $\Delta(\mathcal{A}) = \{p \in \mathbb{R}^{|\mathcal{A}|}_{\geq 0} : \sum_a p_a = 1\}$. We write $V(\mathcal{B})$ for the value of the choice node restricted to actions $\mathcal{B} \subseteq \mathcal{A}$, obtained by taking $q_n \to -\infty$ for $n \notin \mathcal{B}$; in particular, $V(\{a\}) = q_a$ by \cref{ax:vnm}.\footnote{We assume these limits exist and are finite.\label{fn:vblimits}} Similarly, $\pi^{\mathcal{B}}$ denotes the choice rule at the restricted node, with $\pi^{\mathcal{B}}_a$ the probability of selecting $a \in \mathcal{B}$.

\vspace{0.2\baselineskip}
\begin{axiom}[Independence of Irrelevant Alternatives (IIA)]
\label{ax:iia}\
\vspace{-0.2\baselineskip}
\begin{enumerate}[label=(\alph*),itemsep=-2pt]
  \item \emph{Policy IIA.} The ratio $\pi(a \mid q) / \pi(b \mid q)$
        depends only on $q_a - q_b$.
  \item \emph{Value IIA.} For $|\mathcal{B}| \geq 2$, the probability
        $\pi_a^{\mathcal{B}}$ depends only on
        $V(\mathcal{B}) - V(\mathcal{B}\setminus\{a\})$.
\end{enumerate}
\end{axiom}
\vspace{-0.2em}

Part~(a) is a translation-invariant strengthening of Luce's IIA \citep{luce1959individual}, justified by the cardinal scale of \cref{ax:vnm}. Part~(b) links choice probabilities to option values: how often the agent selects an action is determined by that action's marginal contribution to the value of the menu. This is a coherence condition---under~(a), removing an action with selection probability $\pi_a$ always rescales the remaining probabilities by $1/(1-\pi_a)$, so two removals with the same $\pi_a$ have identical effects on choice; (b) requires that such removals also have identical effects on value. In this Axiom, ``depends only on'' means through a single function, independent of the action pair in~(a) and of the menu and its size in~(b). IIA is compelling when actions are genuinely independent, but fails for similar or correlated actions (see \cref{sec:normative}).

\vspace{0.2\baselineskip}
\begin{axiom}[Monotonicity]
\label{ax:mono}\
\vspace{-0.2\baselineskip}
\begin{enumerate}[label=(\alph*),itemsep=-2pt]
\item \emph{Policy monotonicity.} If $q_a > q_b$, then $\pi(a \mid q) > \pi(b \mid q)$.
\item \emph{Value monotonicity.} If $q'_a \geq q_a$ for all $a \in \mathcal{A}$, then $V(q') \geq V(q)$.
\end{enumerate}
\end{axiom}
\vspace{-0.2em}

Part~(a) is a minimal rationality condition: better actions receive higher probability. Part~(b) requires that improving any option cannot decrease the value of the choice node---it implies \emph{menu monotonicity} ($V(\mathcal{A}) \geq V(\mathcal{B})$ for $\mathcal{B} \subseteq \mathcal{A}$), formalizing the design principle that options have non-negative value. Together they provide the regularity needed to determine the policy form and value function from \cref{ax:iia}: (a) makes the ratio function continuous (Step~1); (b) ensures that higher marginal value corresponds to higher choice probability, making the dependence in \cref{ax:iia}(b) invertible (Step~2). 

\paragraph{Remark (Prescriptive use).}\label{rem:prescriptive}
In contrast to \citet{luce1959individual}, who seeks to model human choice, we use \cref{ax:iia,ax:mono} \emph{prescriptively} to constrain what an agent \emph{should} do at choice nodes.

\subsection{Main result}
\label{sec:mainresult}

\begin{theorem}[Boltzmann Representation]
\label{thm:boltzmann}
Let an agent in a finite MDP with $\bar{\gamma} < 1$ and $\vert\mathcal{A}\vert \geq 3$ satisfy VNM preferences over base prospects and DC (\cref{ax:vnm,ax:dc}), with the resulting Bellman relation extended to successor choice nodes as described above. The agent's choice rule and value function satisfy IIA (\cref{ax:iia}) and Monotonicity (\cref{ax:mono}) if and only if there exists $\beta > 0$ s.t.:
\begin{itemize}[leftmargin=2.5em]
\item[\textbf{(i)}]
$\pi(a \mid s) = {\exp(\beta Q(s,a))}/{\sum_{a'} \exp(\beta Q(s,a'))}$
\hfill {\small (Boltzmann policy)}

\item[\textbf{(ii)}]
$\pi(\cdot \mid s) = \arg\max_{\hat\pi \in \Delta(\mathcal{A})} \left[\sum_a \hat\pi_a\, Q(s,a) + \frac{1}{\beta} H(\hat\pi)\right]$
\hfill {\small (entropy representation)}

\item[\textbf{(iii)}] $V(s) = \tfrac{1}{\beta}\mathrm{LSE}(\beta Q(s,\cdot))$, $\;\; Q(s,a) = R(s,a) + \gamma(s,a)\,\mathbb{E}[V(s')]$
\hfill {\small (soft Bellman equation)}
\end{itemize}
For each fixed $\beta>0$, the system (i)--(iii) has a unique solution, denoted $(V^\beta, Q^\beta, \pi^\beta)$.
\end{theorem}

\noindent The assumption $\bar{\gamma} < 1$ is relaxed in \cref{sec:generalized}. $\vert\mathcal{A}\vert \geq 3$ is essential: with two actions, any monotone function of $q_a - q_b$ satisfies Policy IIA; the functional equation in Step~2 also requires three actions. 

\begin{proof}
For the ``if'' direction: (i) gives both Policy IIA and Policy monotonicity directly; (iii) gives Value monotonicity since LSE is componentwise non-decreasing; and (i)+(iii) give Value IIA, since $\pi_a = 1 - \exp(-\beta(V(\mathcal{B}) - V(\mathcal{B}\setminus\{a\})))$. We handle the ``only if'' direction in three steps. 

\textbf{Step 1 (Boltzmann policy).} By \cref{ax:iia}(a), $\pi(a \mid q)/\pi(b \mid q) = \psi(q_a - q_b)$ for some function $\psi$. With $|\mathcal{A}| \geq 3$, ratio consistency $\left(\frac{\pi(a)}{\pi(b)} \cdot \frac{\pi(b)}{\pi(c)} = \frac{\pi(a)}{\pi(c)}\right)$ gives $\psi(d_1)\,\psi(d_2) = \psi(d_1 + d_2)$, which is Cauchy's multiplicative equation. Since $\psi$ is strictly increasing by \cref{ax:mono}(a),  it must be exponential \citep{aczel1966lectures}: $\psi(d) = \exp({\beta d})$ for $\beta > 0$, giving $\pi(a \mid q) \propto \exp({\beta q_a})$. This establishes~(i), with full support since $\exp(\beta q_a) > 0$. Luce applies this same argument to justify the logarithmic Fechnerian scale \citep[Chapter 3]{luce1959individual}, crediting \citet{adams1957approach} for the technical proof.

\textbf{Step 2 (Entropy representation).} By \cref{ax:iia}(b), $\pi_n = g(V(\mathcal{B}) - V(\mathcal{B}\setminus\{n\}))$ for some function $g$. Since the marginal value is non-decreasing in $q_n$ (\cref{ax:mono}(b)) and $\pi_n$ is increasing in $q_n$ (by the softmax form from Step~1), $g$ is increasing. Writing $f = g^{-1}$, the marginal value of action $n$ is $f(\pi_n)$, with $f$ non-decreasing and $f(0) = 0$. We derive two functional equations that pin down $f$.

\emph{Two actions.} Removing each action from $\{1,2\}$ gives $V(\{1,2\}) = q_1 + f(\pi_2) = q_2 + f(\pi_1)$, so $f(\pi_1) - f(\pi_2) = q_1 - q_2$. Substituting $q_1 - q_2 = \tfrac{1}{\beta}\log(\pi_1/\pi_2)$ from Step~1, with $\pi_2 = 1-\pi_1$:
\begin{equation}
f(p) - f(1-p) = \tfrac{1}{\beta}\log\tfrac{p}{1-p}. \label{eq:two-action}
\end{equation}

\emph{Three actions.} Policy IIA implies that removing action $k$ rescales remaining probabilities by $1/(1-\pi_k)$: for instance, $\pi_2^{\{1,2\}} = \pi_2/(1-\pi_3)$. Building $V(\{1,2,3\})$ from $V(\{1\}) = q_1$ by adding 2 then 3 gives $q_1 + f(\pi_2/(1-\pi_3)) + f(\pi_3)$; adding 3 then 2 gives $q_1 + f(\pi_3/(1-\pi_2)) + f(\pi_2)$. Equating:
\begin{equation}
f\!\left(\tfrac{\pi_2}{1-\pi_3}\right) + f(\pi_3) = f\!\left(\tfrac{\pi_3}{1-\pi_2}\right) + f(\pi_2) \quad \text{for all } \pi_2, \pi_3 \in (0,1),\; \pi_2 + \pi_3 < 1. \label{eq:star}
\end{equation}

\emph{Solution.} Define $h(p) := f(p) + \tfrac{1}{\beta}\log(1-p)$. Then \eqref{eq:two-action} gives $h(p) = h(1-p)$, and $f(0) = 0$ gives $h(0) = 0$.
The logarithmic terms cancel in~\eqref{eq:star}, leaving
\begin{equation}
h\!\left(\tfrac{p}{1-t}\right) + h(t) = h\!\left(\tfrac{t}{1-p}\right) + h(p) \quad \text{for all } p, t \in (0,1),\; p + t < 1.  \label{eq:h}
\end{equation}
We show $h \equiv 0$ by proving that any nonzero value would propagate to the boundary, contradicting $h(0) = 0$. Since $h$ differs from $f$ by the continuous function $\tfrac{1}{\beta}\log(1-p)$, $h$ inherits any jump discontinuities from $f$, whose jump sizes are non-negative (\cref{ax:mono}). But since $h(p) = h(1-p)$, a jump at $p$ forces a jump at $1-p$ with opposite sign, from which it follows that $h$ is continuous on $(0,1)$. At the boundary, the assumed submenu limits (footnote~\ref{fn:vblimits}) give $f(p)\to0$ as $p\downarrow0$, so $h(p)\to0=h(0)$. By symmetry, $h(p)\to0$ as $p\uparrow1$, and defining $h(1)=0$ makes $h$ continuous on $[0,1]$. 
By the Extreme Value Theorem, $h$ attains a maximum $M \geq 0$ (since $h(0) = 0$). If $M > 0$, the maximum occurs at some $p_0 \in (0,1)$. Setting $p = p_0(1-p_0)$ and $t = p_0$ in \eqref{eq:h} gives $p/(1-t) = p_0$, so the left side equals $2M$; since each right-side term is at most $M$, both equal $M$, giving $h(p_0(1-p_0)) = M$. Iterating, $p_{n+1} = p_n(1-p_n) \to 0$ with $h(p_n) = M$ for all $n$, so $M=\lim_n h(p_n)=h(0)=0$---contradiction. Hence $h \leq 0$; the same argument for the minimum gives $h \geq 0$, so $h \equiv 0$ and $f(p) = -\tfrac{1}{\beta}\log(1-p)$.

Building up the menu one action at a time, $V(\{1,\ldots,n\}) = q_1 + \sum_{k=2}^n f(\pi_k^{\{1,\ldots,k\}})$. By Policy IIA, restricting to a sub-menu rescales probabilities proportionally, so $\pi_k^{\{1,\ldots,k\}} = \pi_k/(\pi_1+\cdots+\pi_k)$. With $f(p) = -\tfrac{1}{\beta}\log(1-p)$, each term becomes $-\tfrac{1}{\beta}\log\tfrac{\pi_1+\cdots+\pi_{k-1}}{\pi_1+\cdots+\pi_k}$, which telescopes:
\[V(\{1,\ldots,n\}) = q_1 - \tfrac{1}{\beta}\log\pi_1 = \tfrac{1}{\beta}\log\textstyle\sum_a \exp({\beta q_a}) = \tfrac{1}{\beta}\mathrm{LSE}(\beta q).\]
For the Boltzmann policy, this equals $\sum_a \pi_a q_a + \tfrac{1}{\beta}H(\pi)$ (substituting $q_a = \tfrac{1}{\beta}\log\pi_a + \tfrac{1}{\beta}\log Z$ from Step~1, where $Z = \sum_a e^{\beta q_a}$). Thus, the entropy bonus emerges from IIA and monotonicity without assuming any additive optimization structure. Since $\tfrac{1}{\beta}\text{LSE}(\beta q) = \max_{\hat\pi \in \Delta(\mathcal{A})} [\sum_a \hat\pi_a q_a + \tfrac{1}{\beta}H(\hat\pi)]$ with the Boltzmann policy as the unique maximizer, the entropy-regularized objective also provides the soft policy evaluation $V^{\hat\pi}(s) = \sum_a \hat\pi_a Q^{\hat\pi}(s,a) + \tfrac{1}{\beta}H(\hat\pi(\cdot|s))$ used in actor-critic methods.

\textbf{Remark.} \citet[Section~2]{fudenberg2015stochastic} prove the related result that IIA selects Shannon entropy among all strictly concave additive terms. Our derivation does not assume the additive form.

\textbf{Step 3 (Soft Bellman equation).} 
Substituting $V = \tfrac{1}{\beta}\mathrm{LSE}(\beta Q)$ from Step~2 into the Bellman relation, $Q(s,a) = R(s,a) + \gamma(s,a)\,\mathbb{E}[V(s')]$, yields:
\[Q(s,a) = R(s,a) + \gamma(s,a)\,\mathbb{E}_{s'}\!\left[ \tfrac{1}{\beta}\log\textstyle\sum_{a'} \exp({\beta Q(s',a')})\right].\]
The Bellman relation supplies the recursive \emph{structure} ($Q=R+\gamma\,\mathbb{E}[V]$); \cref{ax:iia,ax:mono} supply the \emph{closure} ($V=\tfrac{1}{\beta}\mathrm{LSE}(\beta Q)$).
The resulting fixed-point problem has a unique solution $(V^\beta, Q^\beta, \pi^\beta)$ when $\bar{\gamma} < 1$ (\cref{lem:contraction}).
\end{proof}\vspace{-0.2em}

The two axiom pairs thus govern two distinct kinds of randomness: \cref{ax:vnm,ax:dc} govern \emph{chance} (environmental lotteries over base prospects); \cref{ax:iia,ax:mono} govern \emph{choice} (behavioral selection at decision nodes). Under compound-lottery reduction ($V=\sum_a\pi_aQ_a$), Value IIA fails. For equal-Q menus, removing any action leaves the menu value unchanged ($\Delta V=0$), while its choice probability $1/|\mathcal B|$ varies with menu size, so choice probability is not a function of marginal value. The hard Bellman equation is different: away from ties, it obeys the Value-IIA relation $\pi_a^{\mathcal B}=\mathbf{1}\{\Delta V>0\}$, but it is excluded by strict Policy monotonicity, since distinct suboptimal actions both receive probability zero. The choice between the soft and hard systems is a design decision about whether the agent values its own ability to choose.

\begin{corollary}[Marginal Consistency]
\label{cor:mc}
$\partial V/\partial q_a = \pi(a \mid s)$.
\end{corollary}\vspace{-0.2em}

This follows directly: $\partial V/\partial q_a = \partial_a \tfrac{1}{\beta}\log\sum_b e^{\beta q_b} = \pi(a \mid q)$. Contrast with the deterministic case: $\partial V/\partial q_a = \mathbf{1}\{a = a^*\}$ concentrates all sensitivity on the best action, while the soft value distributes sensitivity smoothly according to the Boltzmann policy.

The following standard result completes the formal development:

\begin{lemma}[Soft Bellman Contraction]
\label{lem:contraction}
For $\gamma(s,a) \leq \bar{\gamma} < 1$, the soft Bellman operator $(\mathcal{T}V)(s) = \frac{1}{\beta}\log\sum_a \exp\!\big(\beta[R(s,a) + \gamma(s,a)\sum_{s'}P(s'|s,a)V(s')]\big)$ is a $\bar{\gamma}$-contraction in $\|\cdot\|_\infty$.
\end{lemma}\vspace{-0.8em}
\begin{proof}
The log-sum-exp is non-expansive: $|\text{LSE}(x) - \text{LSE}(y)| \leq \|x-y\|_\infty$. Therefore $|(\mathcal{T}V_1)(s) - (\mathcal{T}V_2)(s)| \leq \max_a \gamma(s,a)\|V_1 - V_2\|_\infty \leq \bar{\gamma}\|V_1 - V_2\|_\infty$.
\end{proof}

\subsection{Alternative axiomatizations}
\label{sec:alternative}

Value IIA (\cref{ax:iia}b) carries much of the normative weight in \cref{thm:boltzmann} and drives Step~2 of the proof. Read forward, Value IIA says that the probability of selecting an action depends only on the loss in menu value caused by removing it, not on the composition of the menu producing that loss. Under monotonicity it can equivalently be read backward. Policy IIA implies that removals with the same $\pi_a$ have the same effect on choice. Each rescales the remaining probabilities by $1/(1-\pi_a)$. Value IIA requires them to have the same effect on value. Without such a condition, the policy and value function remain structurally decoupled. Related work links menu evaluation and choice probabilities qualitatively \citep{ahn2013preference,fudenberg2015dynamic}, but does not make choice probability a function of an option's cardinal marginal value.

The same coupling can be imposed in two alternative ways: directly through marginal consistency or structurally through an additive form.

\textbf{Route 1: Marginal consistency.} Marginal consistency (\cref{cor:mc}), $\partial V/\partial q_a=\pi_a$, directly implies Value Monotonicity because $\pi_a>0$. The submenu limit and Step~1 softmax policy give $\Delta V_a=\int_{-\infty}^{q_a}\pi_a^{\mathcal{B}}(x,q_{-a})\,dx=-\tfrac{1}{\beta}\log(1-\pi_a^{\mathcal{B}})$. Equivalently, $\pi_a^{\mathcal{B}}=1-e^{-\beta\Delta V_a}$, which is Value IIA. The two value packages are therefore equivalent given the remaining assumptions.

\textbf{Route 2: Assume the additive form.} Suppose instead that menu value has the additive perturbed utility (APU) form of \citet{fudenberg2015stochastic}, $V = \max_{\hat\pi \in \Delta(\mathcal{A})} [\hat\pi \cdot Q - \sum_a c(\hat\pi_a)]$. This is a separable instance of the regularized-MDP framework \citep{geist2019theory}. Given the Step~1 softmax policy, the logit selection result of \citet[Theorem~2]{fudenberg2015stochastic} applies on our full Q-vector domain. Together with the fixed cardinal scale and singleton normalization, it reduces the additive representation to the entropy regularizer $c(p)=\tfrac{1}{\beta}p\log p$, thereby recovering the same representation.

Thus, given the remaining assumptions, Value IIA (plus Value Monotonicity), marginal consistency, and the additive form are interchangeable formulations of the same policy--value coupling. We retain the qualitative route in \cref{thm:boltzmann} because it states the policy--value commitment directly at the menu level without presupposing differentiability or an additive structure.

\section{Consequences for Reinforcement Learning}
\label{sec:consequences}

We develop three consequences of the axiomatic framework: convergence of the soft Bellman equation under generalized discounting (\cref{sec:generalized}), monotonicity of actual return in~$\beta$ (\cref{sec:performance}), and a numerical illustration that IIA selects Shannon entropy within the Tsallis family (\cref{sec:tsallis}).

\subsection{Generalized discounting and the soft Bellman equation}
\label{sec:generalized}

\Cref{thm:boltzmann} assumes $\bar{\gamma} < 1$. We now relax this. \citet{white2017UnifyingTS} established weighted-norm contraction results for Bellman operators with transition-based $\gamma(s,a,s') \in [0,1]$. \citet{pitis2019rethinking} derived $\gamma(s,a)$ from preferences, allowing $\gamma(s,a) > 1$ for certain amplifying actions (e.g., investments, compounding effects), and showed convergence of stationary-policy evaluation under a \emph{horizon continuity} axiom and a spanning condition.

Extending this analysis to entropy-regularized policies reveals a new phenomenon. Because Boltzmann policies have full support, amplifying actions always receive positive weight---even when the deterministic optimum avoids them entirely. This creates a potential failure mode:

\begin{proposition}[Soft Bellman may diverge given amplifying actions]
\label{prop:diverge}
There exist MDPs where the deterministic optimal value is finite but the
soft Bellman equation has no finite fixed point for sufficiently low $\beta$.
\end{proposition}

\begin{wrapfigure}{r}{0.40\textwidth}
\vspace{-1.5em}
\centering
\includegraphics[width=0.4\textwidth]{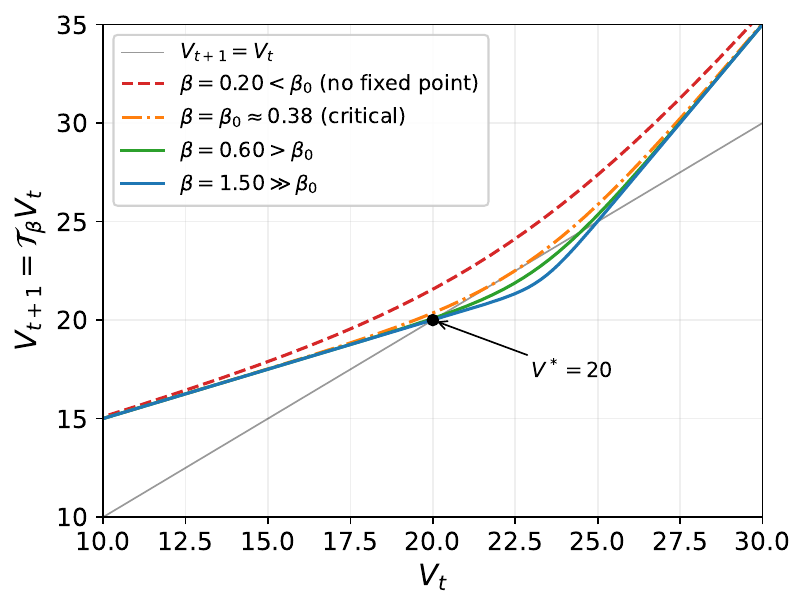}
\vspace{-2.5em}
\end{wrapfigure}
\noindent\textit{Proof.}
Consider a single-state MDP with two actions, \emph{safe} ($R = 10$,
$\gamma = 0.5$) and \emph{risky} ($R = -25$, $\gamma = 2$).
The deterministic optimum avoids \emph{risky} and $V^* = 20$. For the soft Bellman operator, however, we have $f(V) = \tfrac{1}{\beta}\text{LSE}(\beta(10+0.5V),\, \beta(-25+2V))$ and there is a critical threshold $\beta_0 = \frac{1}{5}\ln\frac{27}{4} \approx 0.382$, such that for
$\beta < \beta_0$, $f(V) > V$ everywhere and no fixed point exists (see figure). For $\beta > \beta_0$,
a stable fixed point exists near the deterministic optimum. For $\beta < \beta_0$, the operator lies strictly above
the identity:
amplifying actions receive enough weight to prevent convergence. $\beta_0$ thus defines a minimum rationality threshold. \qed

\vspace{0.5em}
\textbf{Resolution via horizon continuity.}\looseness=-1{} The axiomatic framework of \citet{pitis2019rethinking} derives $\gamma(s,a)$ from preferences. Its \emph{horizon continuity} axiom, which says that the far future eventually ceases to matter under any policy, implies $\rho(\Gamma^\pi T^\pi) < 1$ for all stationary $\pi$, provided the statewise value vectors induced by policies span $\mathbb{R}^{|\mathcal{S}|}$ \citep[Theorem~4]{pitis2019rethinking}, where $\rho$ denotes spectral radius and $(\Gamma^\pi T^\pi)(s,s') = \sum_a \pi(a|s)\,\gamma(s,a)\,P(s'|s,a)$ is the discount-weighted transition matrix. Since this holds for every Boltzmann policy, the $\beta_0$ threshold is eliminated:

\begin{proposition}[Horizon continuity resolves $\beta_0$]
\label{prop:horizon}
If $\rho(\Gamma^\pi T^\pi) < 1$ for all stationary $\pi$ (as implied by
horizon continuity with finite $|\mathcal{S}|$ and the spanning condition above), then the soft Bellman
equation has a unique fixed point for all $\beta > 0$, and soft policy
iteration converges to it.
\end{proposition}\vspace{-0.5em}
\begin{proof}[Proof sketch]      
Let $V^\pi$ denote the entropy-regularized value of stationary policy $\pi$. The proof follows standard soft policy iteration, with the Neumann series $(I - \Gamma^\pi T^\pi)^{-1} = \sum_{k \geq 0}(\Gamma^\pi T^\pi)^k \geq 0$
(valid since $\Gamma^\pi T^\pi$ is entrywise non-negative and $\rho(\Gamma^\pi T^\pi) < 1$) replacing the usual contraction argument.
This non-negative inverse preserves the soft policy improvement inequality, giving $V^{\pi'} \geq V^\pi$, where $\pi'$ is the Boltzmann improvement of $\pi$.
Because $\pi \mapsto \Gamma^\pi T^\pi$ is continuous on the compact policy simplex and $\rho(\Gamma^\pi T^\pi) < 1$ throughout, the policy-evaluation resolvents are uniformly bounded. Together with bounded one-step regularized rewards, this bounds the monotone policy-value sequence, which therefore converges. Its policy-improvement residual vanishes in the limit, which is therefore a fixed point by continuity. 
For uniqueness, note that any fixed point $V$ satisfies $V \geq V^\pi$ for all $\pi$ by the same resolvent argument, and since $V = V^{\pi_V}$ for its induced Boltzmann policy, any two fixed points satisfy $V_1 \geq V_2 \geq V_1$.
\end{proof}

\subsection{Performance, sensitivity, and exploration}
\label{sec:performance}

Let $J(\pi, s) = \mathbb{E}_\pi[\sum_t (\prod_{k<t}\gamma(s_k,a_k))\, R(s_t, a_t) \mid s_0 = s]$ denote the actual (unregularized) expected return of policy $\pi$, let $\pi^*$ denote a deterministic optimal policy with respect to $J$, and recall that $\pi^\beta$ is the Boltzmann policy of \cref{thm:boltzmann}. 
The log-sum-exp satisfies $\max_a Q_a \leq \tfrac{1}{\beta}\text{LSE}(\beta Q) \leq \max_a Q_a + \frac{1}{\beta}\log|\mathcal{A}|$, giving a performance gap $J(\pi^*) - J(\pi^\beta) \leq \log|\mathcal{A}| / (\beta(1-\bar{\gamma}))$ that vanishes as $\beta \to \infty$ \citep[see][for sharp bounds]{muller2024sharp}. A natural question is whether higher $\beta$ always improves actual (not soft) return:

\begin{proposition}[Monotonicity of actual return]
\label{prop:monotonicity}
For any finite MDP with $\bar{\gamma} < 1$, the actual return $J(\pi^\beta, s)$ is
non-decreasing in $\beta$ for all $s$.
\end{proposition}
\vspace{-1.5em}
\begin{proof}
Let $\alpha = 1/\beta$ and fix any start state, suppressed below ($\pi^\alpha$ is optimal from every start state at once). The Boltzmann policy maximizes
$\pi^\alpha = \arg\max_\pi [J(\pi) + \alpha\,\bar{\mathcal{H}}(\pi)]$,
where
$\bar{\mathcal{H}}(\pi) = \mathbb{E}_{\pi}[\sum_{t \geq 0} (\textstyle\prod_{k<t}\gamma(s_k,a_k))\, H(\pi(\cdot|s_t))]$, matching the discounting of $J$.
For $\alpha_1 < \alpha_2$ (i.e., $\beta_1 > \beta_2$), with $\pi_i=\pi^{\alpha_i}$, optimality of each policy at its own temperature gives
$J(\pi_1) - J(\pi_2) \geq \alpha_1[\bar{\mathcal{H}}(\pi_2) - \bar{\mathcal{H}}(\pi_1)]$ and $J(\pi_1) - J(\pi_2) \leq \alpha_2[\bar{\mathcal{H}}(\pi_2) - \bar{\mathcal{H}}(\pi_1)]$.
These imply $\bar{\mathcal{H}}(\pi_2) \geq \bar{\mathcal{H}}(\pi_1)$. 

Substituting back:
$J(\pi_1) - J(\pi_2) \geq \alpha_1 \cdot [\bar{\mathcal{H}}(\pi_2) - \bar{\mathcal{H}}(\pi_1)] \geq 0$.
The argument extends to any convex regularizer \citep[cf.][]{milgrom2002envelope}.
\end{proof}\vspace{-0.3em}

\textbf{Comparison of exploration strategies.} \Cref{tab:strategies} compares common exploration strategies against Policy IIA, monotonicity, and the additive perturbed utility (APU) property of \citet{fudenberg2015stochastic} (\cref{sec:alternative}). Of the listed strategies, only Boltzmann satisfies all properties. Mellowmax \citep{asadi2017alternative} induces a Boltzmann policy (satisfying Policy IIA), but its log-mean-exp value operator violates Value IIA. Tsallis/sparsemax policies \citep{lee2018sparse} violate IIA at the policy level, which the next subsection quantifies.

\begin{table}[htbp]
\caption{Axiomatic comparison of exploration strategies.}\vspace{-1.2em}
\label{tab:strategies}
\begin{center}
\begin{tabular}{lcccc}
\toprule
\textbf{Strategy} & \textbf{Policy IIA} & \textbf{Mono.} & \textbf{Full support} & \textbf{APU} \\
\midrule
Boltzmann (softmax) & $\checkmark$ & $\checkmark$ & $\checkmark$ & $\checkmark$ \\
$\varepsilon$-greedy & $\times$ & $\times$ & $\checkmark$ & $\times$ \\
Thompson sampling & $\times$ & $\checkmark^*$ & $\checkmark^*$ & $\times$ \\
Tsallis/sparsemax ($q>1$) \citep{lee2018sparse} & $\times$ & $\times^\dagger$ & $\times$ & $\times$ \\
\bottomrule
\end{tabular}
\end{center}
\vspace{-0.3em}
{\small $^*$Both properties hold for independent location-family posteriors with a common full-support noise distribution (e.g., equal-variance Gaussians), but neither is guaranteed in general.\\ $^\dagger$Holds on the active support; fails globally (excluded actions with distinct Q-values receive equal probability).}
\end{table}

\textbf{Example: Thompson sampling violates IIA.}\quad
Consider three actions with independent discrete posteriors:
$\mu_a \in \{2, 5\}$ each with probability $\frac{1}{2}$,
$\mu_b \in \{0, 4\}$ each with probability $\frac{1}{2}$, and
$\mu_c = 3$ (deterministic). In menu $\{a, b\}$, $a$ wins in three of four equiprobable outcomes: $\pi(a)/\pi(b) = 3$. In menu $\{a, b, c\}$:
$(2,0,3) \to c$, $(2,4,3) \to b$, $(5,0,3) \to a$, $(5,4,3) \to a$, giving $\pi(a)/\pi(b) = 2 \neq 3$.

\subsection{Visualizing IIA across the Tsallis family}
\label{sec:tsallis}

\begin{wrapfigure}{r}{0.4\textwidth}
\vspace{-1.4em}
\centering
\includegraphics[width=0.39\textwidth]{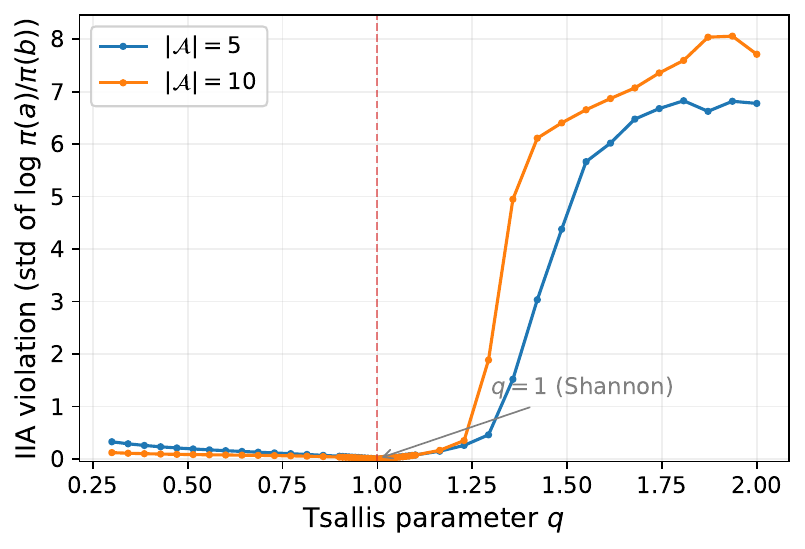}
\vspace{-2em}
\end{wrapfigure}

\looseness=-1 Step~2 of \cref{thm:boltzmann} establishes that the value function is uniquely determined by IIA and monotonicity, with Shannon entropy emerging as the unique consistent bonus. The figure quantifies IIA violations along the Tsallis entropy family $H_q(\pi) = \frac{1}{q-1}(1 - \sum_a \pi_a^q)$, with Shannon entropy at $q \to 1$ and sparsemax at $q = 2$ \citep{lee2018sparse}. For each $q$, we fix $Q_a=1$ and $Q_b=0$, draw the remaining Q-values 500 times from $\mathrm{U}[-3,3]$, and report the standard deviation of $\log(\pi(a)/\pi(b))$ at $\beta=1$ over draws in which both actions remain in the support. Only $q = 1$ achieves zero violation. For $q > 1$, compact support makes $\log(\pi(a)/\pi(b))$ diverge as either action approaches exclusion. This demonstrates that sparsemax and other Tsallis alternatives are not minor departures from IIA but qualitatively different choice rules.

\section{Discussion}
\label{sec:discussion}

\subsection{Cross-disciplinary foundations}
\label{sec:convergence}

\begin{table}[!t]
\caption{Cross-disciplinary derivations and motivations for Boltzmann rationality.\vspace{-1em}}
\label{tab:synthesis}
\makebox[\textwidth][c]{%
\footnotesize
\begin{tabular}{lll}
\toprule
\textbf{Tradition} & \textbf{Key Axiom / Assumption} & \textbf{Yields} \\
\midrule
Choice theory \citep{luce1959individual} & IIA & Boltzmann policy (static) \\
Random utility \citep{yellott1977relationship} & RUM + IIA & Gumbel $\Leftrightarrow$ Boltzmann \\
Information theory \citep{shore1980axiomatic} & Consistent inference & MaxEnt $\to$ Boltzmann \\
Bounded rationality \citep{ortega2013thermodynamics} & KL info cost & Free energy $\to$ Boltzmann \\
Rational inattention \citep{matejka2015rational} & Shannon MI cost & Boltzmann from MI cost \\
Perturbed utility \citep{fudenberg2015stochastic} & APU + IIA & Entropy cost $\Leftrightarrow$ Boltzmann \\
\midrule
Robustness \citep{eysenbach2022maximum} & Adversarial perturbation & Regularization $\to$ robustness \\
Compositionality \citep{haarnoja2018composable} & Exponential form & Skill composition via Q-addition \\
Econometrics \citep{rust1987optimal} & Gumbel shocks + Bellman & Soft Bellman via inclusive value \\
Online learning \citep{auer2002nonstochastic} & Minimax regret & Exponential weights \\
\midrule
\textbf{This paper (\cref{thm:boltzmann})} & \textbf{VNM + IIA + Monotonicity} & \textbf{Boltzmann policy + soft Bellman} \\
\bottomrule
\end{tabular}%
}
\end{table}

\Cref{tab:synthesis} summarizes several independent derivations of and motivations for the Boltzmann structure.
The first six rows provide independent axiomatic derivations, with IIA as the common thread linking choice theory, random utility, and perturbed utility; information theory and rational inattention arrive at the same form through entropy-maximization and information-cost principles. Robustness, compositionality, and online learning provide practical motivation but not uniqueness: each applies to broader families of regularizers or policies. The econometric tradition independently derives the soft Bellman equation via Gumbel shocks, predating its use in RL by decades. The RL literature has largely developed separately from these foundations, adopting the softmax form without a normative justification.

\vspace{-0.5\baselineskip}
\subsection{Normative assessment}
\label{sec:normative}

\textbf{When to use Boltzmann.} Multiple independent axiom systems single out the Boltzmann form. In random utility theory, \citet{yellott1977relationship} shows that, for additive i.i.d.\ noise and three or more alternatives, IIA holds if and only if the unobserved value noise is Gumbel-distributed. Since the Gumbel distribution generates the softmax, it follows that any IIA-consistent noise model in this class produces Boltzmann choice probabilities. In information theory, the Shore--Johnson axioms uniquely justify KL divergence as the consistent inference functional \citep{shore1980axiomatic}, and \citet{matejka2015rational} show that Shannon mutual-information cost yields the multinomial logit.

That these independent lines of argument---perceptual noise, information cost, and the axiomatic characterization of the present paper---converge on the same functional form through different axiom systems is evidence of the Boltzmann form's structural importance. IIA also offers a concrete engineering virtue: actions can be added to or removed from the action space without recalibrating existing relative preferences. This modularity is valuable for compositional or evolving agent architectures. If VNM preferences, IIA, and monotonicity are all appropriate for the domain, then Boltzmann is not merely a convenient choice---it is the \emph{unique} consistent policy, and the design problem reduces to choosing $\beta$. A practical diagnostic: fix two actions $(a,b)$ and check whether $\log(\pi(a|s)/\pi(b|s))$ varies when other actions are added or removed; under IIA, this ratio is invariant.

\textbf{When not to use Boltzmann.} IIA's main failure mode is the
\emph{similarity effect} \citep{debreu1960individual}, popularized as the ``red bus / blue bus'' problem. Consider an LLM choosing among 11 candidate responses with \emph{equal} Q-values: 10 paraphrases of one response, and 1 semantically distinct response. Under IIA, each paraphrase is treated as an independent alternative, so the Boltzmann policy assigns
$\approx$91\% total probability to paraphrases and only ${\approx}9\%$ to the distinct response, a distortion produced by IIA together with treating surface-distinct responses as independent alternatives \citep{holtzman2021surface}. Whenever the action space has correlated structure, IIA wastes probability mass on near-duplicates. The nested logit model \citep{mcfadden1978modeling} addresses this by partitioning actions into nests $\mathcal{N}_1, \ldots, \mathcal{N}_K$. Within each nest, $\pi(a \mid \mathcal{N}_k) \propto \exp(\beta_k Q_a)$, so IIA holds locally with nest-specific temperature $\beta_k$. Nest selection follows $\pi(\mathcal{N}_k) \propto \exp(\mu \cdot \tfrac{1}{\beta_k}\text{LSE}(\beta_k Q_{\mathcal{N}_k}))$, where $\mu$ governs cross-nest substitution. When $\mu = \beta_k$ for all $k$, this reduces to standard Boltzmann; when $\mu < \beta_k$, similar actions within a nest cannibalize each other's probability rather than drawing from dissimilar alternatives.

\textbf{KL regularization and RLHF.} In practice, the most common formulation is $\max_\pi \sum_a \pi_a Q_a - \frac{1}{\beta}D_{\text{KL}}(\pi \| \pi_0)$, yielding $\pi(a) \propto \pi_0(a)\exp(\beta Q_a)$---a KL penalty from a reference policy $\pi_0$ rather than a pure entropy bonus. (In the pairwise comparisons typical of RLHF preference elicitation, $\vert\mathcal{A}\vert = 2$ and IIA does not force the exponential form; see \cref{sec:limitations}.) This is the form used in RLHF \citep{ouyang2022training}, rational inattention \citep{matejka2015rational}, and linearly-solvable MDPs \citep{todorov2006linearly}. The corresponding axiom would be a version of \emph{generalized IIA}: $\log(\pi(a)/\pi(b)) = \beta(q_a - q_b) + \log(\pi_0(a)/\pi_0(b))$, which is Luce's choice axiom with scale values $v(a) = \pi_0(a)e^{\beta q_a}$.

\citet{xu2023rlhf} show that IIA induces perverse incentives in RLHF: a reward model trained under the Bradley-Terry assumption (which assumes IIA) systematically misranks responses when one category has many near-duplicates in the candidate set. The effect is compounded when preferences are aggregated across heterogeneous annotators, where Arrow-type impossibilities arise \citep{maura2025jackpot}. The nested logit remedy---grouping semantically similar responses into nests while preserving IIA within nests---applies regardless of how preferences are elicited. Developing RLHF with nested logit or cross-nested preferences is (to our knowledge) an open direction.

\textbf{$\beta$-reward scale duality.} No axiomatization determines $\beta$. Scaling rewards by $\alpha$ scales Q-values by $\alpha$, so $\pi^\beta_R \equiv \pi^{\beta/\alpha}_{\alpha R}$ because $(\beta/\alpha)(\alpha Q)=\beta Q$. Consequently, choice data identify only the product of inverse temperature and reward scale. In RLHF, the learned reward scale $\alpha$ and the KL coefficient $1/\beta$ therefore matter only through their ratio, $\alpha/(1/\beta)=\alpha\beta$. Tuning both separately is redundant: a grid search over $(\beta,\alpha)$ reduces without loss to a one-dimensional sweep over $\alpha\beta$ \citep[cf.][Eq.~4, which uses $\beta$ for the KL coefficient]{rafailov2023direct}. More broadly, any claim about the ``right'' temperature is vacuous without fixing the reward scale.

\vspace{-0.5\baselineskip}
\subsection{Limitations and open questions}
\label{sec:limitations}

\textbf{What does the theorem rationalize?} Like the VNM Expected Utility representation in \cref{ax:vnm}, \cref{thm:boltzmann} characterizes a functional form through qualitative axioms. It addresses the question raised by \citet[Remark~5.3]{pitis2023consistent} by replacing ``why softmax and the entropy bonus?'' with ``why Policy and Value IIA?'' Its normative force therefore depends on how compelling these axioms are. While they precisely relate choice probabilities to option values, they do not explain why choice itself is valuable (\cref{sec:motivation} takes that value as a premise). Nor do familiar motivations for stochasticity based on partial observability, exploration, or robustness determine softmax, the entropy bonus, or $\beta$.

One possible path starts from the hard Bellman maximum. Under suitable independence and tail conditions, repeated maximization of noisy value estimates may produce approximately Gumbel-distributed errors \citep{garg2023extreme}. Independent additive Gumbel errors would in turn yield softmax choice and, up to an additive constant, the LSE expected maximum \citep{yellott1977relationship,mcfadden1981econometric}. Making this argument exact in a sequential setting would require justifying the noise model and its propagation through Bellman recursion. A fuller account would derive the precise policy--value coupling, including its scale relative to rewards from a normatively compelling model.

\textbf{Binary actions.} The axiomatization requires $\vert\mathcal{A}\vert \geq 3$: with two actions, any monotone function of $q_a - q_b$ satisfies IIA, so the exponential form is not uniquely determined. In binary settings (e.g., pairwise preference elicitation in RLHF), softmax remains natural but is not axiomatically forced by IIA alone. Similarly, Value IIA does not force a unique value function with two actions, as the three-action functional equation \eqref{eq:star} is needed for uniqueness in Step 2 of the proof of \cref{thm:boltzmann}.

\textbf{Continuous action spaces.}\looseness=-1{} The characterization might be extended to continuous $\mathcal{A} \subseteq \mathbb{R}^d$ with a reference measure $\mu$: IIA would become a density-ratio condition, and the Cauchy argument in Step~1 is dimension-free, yielding $\pi(a \mid s) \propto \exp({\beta Q(s,a)})$ relative to $\mu$. Steps 2 and 3 might also be extended with differential entropy. The choice of $\mu$, however, is an additional modeling input---Soft Actor-Critic \citep{haarnoja2018soft} uses Lebesgue measure in its chosen action coordinates. Given $\mu$, IIA requires the $\mu$-density ratio $\pi(a)/\pi(b)$ to depend only on $Q(a) - Q(b)$.

\textbf{Static vs.\ dynamic IIA.}\looseness=-1{} We apply IIA state-by-state, with VNM and dynamic consistency providing the intertemporal structure. A natural stronger condition is \emph{dynamic IIA}: the relative probability of two trajectories should not depend on what other trajectories are available. Whether static IIA implies dynamic IIA and whether additional conditions such as the recursivity axiom of \citet{fudenberg2015dynamic} are needed are open questions.

\vspace{-0.5\baselineskip}
\subsection{Conclusion}

\looseness=-1 IIA and monotonicity, applied to both policy and value within the MDP framework, jointly determine the Boltzmann policy, Shannon entropy bonus, and soft Bellman equation. The key modeling decision---restricting VNM Independence to base prospects while imposing IIA and monotonicity at choice nodes---separates environmental chance from agent choice, providing a first-principles account of the default functional form. The characterization also makes explicit when the Boltzmann form should \emph{not} be used: in the presence of correlated actions that invalidate IIA.

\subsubsection*{Acknowledgments}
The author thanks the anonymous reviewers and area chair for their careful reading and constructive suggestions. AI assistants (Anthropic's Claude, Google's Gemini, and OpenAI's GPT) assisted with literature review, brainstorming, technical discussion, proofs, draft critique, and editing. The author verified all results and takes full responsibility for the content. This work was supported in part by a CIFAR AI Safety Fellowship and an OpenAI Superalignment Fast Grant.

\bibliography{main}
\bibliographystyle{rlj}
\end{document}